\documentclass[a4paper,twoside,french,apacite]{article}

\usepackage{epstopdf}
\usepackage[caption=false]{subfig}

\usepackage[natbibapa,nodoi]{apacite}
\usepackage[utf8]{inputenc}
\usepackage[T1]{fontenc}
\usepackage{times}
\usepackage{hyperref}
\usepackage{color}
\usepackage{amssymb}
\usepackage{graphicx}
\usepackage{enumitem}
\usepackage{url}
\usepackage[small]{caption}
\usepackage{amsmath}
\usepackage{amsthm}
\usepackage{booktabs}
\usepackage{algorithm}
\usepackage{algorithmic}
\usepackage{color}
\usepackage{todonotes}
\usepackage{a4wide}
\usepackage{soul}
\newcommand\todoin[2][]{\todo[inline, caption={2do}, #1]{
\begin{minipage}{\textwidth-4pt}#2\end{minipage}}}

\renewcommand{\L}{\mathcal{L}}
\usepackage{environ} 
\NewEnviron{omettre}{}

\hypersetup{
    pdftitle={Belief Revision and Incongruity: is it a joke?},    
    pdfauthor={Dupin de Saint-Cyr, Prade},     
    pdfsubject={Artificial Intelligence},   
}

\theoremstyle{plain}
\newtheorem{theo}{Theorem}[section]

\newtheorem{prop}[theo]{Proposition}

\theoremstyle{definition}
\newtheorem{defi}[theo]{Definition}
\newtheorem{ex}[theo]{Example}

\theoremstyle{remark}
\newtheorem{remark}{Remark}

\usepackage{tikz}
\usetikzlibrary{shapes.geometric,arrows,arrows.meta,decorations.markings,intersections}

\newcommand{\entails}{\mid\!\sim}

\makeatletter
\newcommand{\dotminus}{\mathbin{\text{\@dotminus}}}

\newcommand{\@dotminus}{%
  \ooalign{\hidewidth\raise1ex\hbox{.}\hidewidth\cr$\m@th-$\cr}%
}
\makeatother
\newenvironment{exs}[1]{\smallskip \noindent{\bf Example~\ref{#1}   (continued):}\em}{}
\date{} 
\title{Belief Revision and Incongruity: is it a joke?\thanks{A special paper on/in humor/honor for/of Philippe Besnard. \\
This is an Accepted Manuscript of an article published by Taylor \& Francis in Journal of Applied Non-Classical Logics on August 2023, available at: https://doi.org/ 10.1080/11663081.2023.2244379.}}

\author{Florence Dupin de Saint-Cyr \and Henri Prade\thanks{Contact: \texttt{florence.bannay@irit.fr} and \texttt{henri.prade@irit.fr}}}
\date{IRIT-CNRS, Université de Toulouse,\\
118, route de Narbonne, 31062 Toulouse Cedex 9, France\\
August 2023\\
}

\begin{document}
\maketitle

\bigskip


\bigskip

 \begin{abstract}
Incongruity often makes people laugh. You have to be smart to say stupid things. It requires to be even smarter for understanding them. This paper is a shameless attempt to formalize this intelligent behavior in the case of an agent listening to a joke. All this is a matter of revision of beliefs, surprise and violation of norms.   
 
\end{abstract}

  \quad\quad\quad\quad\quad\quad\quad\quad   \quad\quad\quad\quad\quad\quad\quad\quad  \quad\quad\quad\quad\quad
  \quad\quad\quad\quad
  ``Incongruity is never superfluous''
  
        \quad\quad\quad\quad\quad\quad\quad\quad\quad\quad\quad\quad\quad\quad\quad\quad\cite{TheaArbee} \emph{Is the superfluous a (new) modality?}\footnote{Thea Arbee aka Léa Sombé (``Les A sont B'' / ``The A are B") is the pen name of a group of researchers active in the early 1990's to which Philippe took part. It happened that Léa Sombé embodied herself during brief appearances to celebrate  members of the group. Léa Sombé sends her best kisses to Philippe. }



\section{Introduction}
Even if much has been written about ingredients that trigger laughter, researchers are still far from having completely understood their interplay in the cognitive process that leads a listener to guffaw at a pun or a joke. They are even farther from a detailed analysis and modeling of the mechanisms that are at work in this process.  

However, in  recent articles \cite{DuPr2020,DuPr2022}   took a first step in this direction by laying bare that a belief revision mechanism was solicited in the reception of a narrative joke. Namely the punchline, which triggers a revision, is both \emph{surprising} and \emph{explains perfectly} what was reported in the beginning of the joke. A similar idea has been more informally proposed in \cite{Ritchie2002}. It is quite clear that this is insufficient for characterizing a narrative joke. One may experience similar surprise and revelation effects when reading a detective story or when following the path of a scientific discovery, while there is usually nothing to laugh about.

This latter observation suggests that there are other ingredients in a joke. Among possible candidates that may play a role, incongruity  has been often identified. 
Incongruity has a rather broad meaning, since it may cover anything contrary to custom or good manners. This ranges from unexpected behavior to absurd or crazy reasoning to violations of taboos. This is a rather large scope of possibilities which have in common the non-respect of some (social) norms. In this paper we look at how to incorporate this ingredient into our model after rediscussing the defeasible setting in which it takes place.

The rest of the article  is structured as follows. Section \ref{Ing} provides an overview of the main ingredients of fun, as discussed by philosophers and psychologists, emphasizing the role of incongruity.  Section \ref{formalisation} restates, details and exemplifies the
authors'model based on belief revision and surprise, using nonmonotonic reasoning methods, as recalled in Section \ref{bg}.
Section \ref{Incong}  discusses how to add incongruity by taking into account norms. The conclusion 
points out future lines of research for a better coverage of the multiple aspects involved in humor.


\section{Incongruity among other ingredients of fun}\label{Ing}
 In the philosophical and psychological literature on humor, three ideas have often been put forward to trigger laughter: surprise, incongruity and superiority (see \cite{Keith-Spiegel1972} for an historical account).

\begin{itemize}
\item \emph{Surprise}: Surprise has been considered as a necessary condition for laughter by many authors and for a long time. \cite{Descartes1649} in \emph{Les Passions de l'Âme} (Passions of the Soul) (sections CXXIV to CXXVI)   analyzes the physiological aspect of laughter and its causes by underlining that it results from a surprise (related to ``Admiration") occurring in a context of (moderate) ``Joye". 
In a psychological study, \cite{Hollingworth1911}  shows that when novelty or surprise is eliminated then the response to a humorous stimulus is diminished. The importance of surprise in the analysis of humor is acknowledged by Marius  \cite{Latour1956} who dedicated the fourth essay of his 1956 book on laughter to ``la surprise source du rire'' (surprise as a source of laughter).

\item \emph{Incongruity}: Incongruity is also advocated as a necessary ingredient of humor, e.g., \cite{Ritchie1999}. It is often associated to surprise. 
 \cite{Beattie1776}  notes very early that humor can be based on ideas/situations that are  divergent from normal customs, or based on inconsistent and inappropriate situations. 
 \cite{Kant1790}, as  \cite{Schopenhauer1819}, speaks about the sudden incongruity of the perception of a real object compared to what had been imagined/expected. 
In his  book  \emph{Laughter: An Essay on the Meaning of the Comic},  \cite{Bergson1900}, advocates that the source of laughter involve an over-rigid response, that would, to use Bergson's terms, ``superimpose the mechanical on the living''; this may be viewed as a particular case of incongruity.   \cite{Delage1919} sees incongruity as a disharmony which is an essential condition of the comic: ``Il doit exister entre l'effet produit et sa cause une désharmonie quantitative ou qualitative, d'où résulte une impression de surprise, un effet d'imprévu'' (There must exist between the effect produced and its cause a quantitative or qualitative disharmony, from which results an impression of surprise, an effect of the unexpected).
 Other references and more details on the theory of incongruity can be found in the collection of papers edited by   \cite{ChapmanFoot1976}.
\item \emph{Superiority} : 
According to F. R. \cite{Fleet1890}, referring to the \emph{Leviathan} \cite{Hobbes}, ``Laughter is a sudden glory arising from some sudden conception of some eminency in ourselves by comparison with the inferiority of others''. Thus what is laughable would be the highlighting of an imperfection (absent in the listener) by the one who wants to make people laugh, arousing a sense of superiority. The authors of this theory see laughter as a victory over other people (or circumstances). One feels superior by laughing because one is less stupid, less ugly, less unlucky, less weak than the target of mockery.
One may consider that the notion of superiority underlies \cite{Bergson1900}'s famous book on laughter, notably through the theatrical aspect of his examples (where the audience is laughing at the characters embodied by the actors). In the same vein, Marcel \cite{Pagnol47}, a well-known French play-writer  has written a short essay on this topic where he emphasized that ``Laughter is a song of triumph (chant de triomphe): it is the expression of a momentary superiority, but suddenly discovered, of the laugher on the mocked. There are two kinds of laughter'', one considered by Pagnol as positive (``healthy, invigorating, relaxing'') where one ``feels superior to the whole world or to oneself'', the other negative which is ``hard and almost sad'', where one laughs because the other is inferior to oneself.
\end{itemize}

These three elements have also been mentioned by other authors, see \cite{AtRa1991,Ritchie2018}, sometimes together with more ingredients such as ``double entendre'', ``lack of compassion'', ``situation matching'', etc. For instance,  \cite{Maier1932} summarises his view as: ``The thought-configurations which makes for a humorous experience must (1) be unprepared for; (2) appear suddenly and bring with it a change in the meaning of its elements; (3)  be made up of elements which are experienced entirely objectively [...] ; (4) contain as its elements the facts
appearing in the story, and these facts must be harmonized, explained and unified; and (5)  have the characteristics of the ridiculous in that its harmony and logic apply only to its own elements'' [pp.73-74].

In this paper, we would like to continue our exploration about joke analysis by taking into account incongruity as a new part of the mechanism introduced in \cite{DuPr2020} where surprises were handled in a logical setting  involving nonmonotonic reasoning. We leave the superiority ingredient for further research since it involves a richer framework able to represent the beliefs of several agents (the teller, the listener and the other people that are mocked).
We note that the notion of ``incongruity resolution’’ plays an important role in psychology literature, see  \cite{ChapmanFoot1976,Bariaud1983,Ritchie1999}: this expression means finding a logic or at least find the initial thing less incongruous than it seemed, i.e., ``make it acceptable'' (inasmuch the matter is not perceived as serious or dangerous, in particular  with a young audience). The French psychologist Françoise Bariaud’s essay about the genesis of humor in children (\cite{Bariaud1983}) claims  that the notion of cognitive reference frame (which the child is in the process of constructing) is used both to detect the incongruity and to resolve it (calling upon a logical sense). She also notes that it is not necessary to put taboos into play, as we can exemplify by the following riddle: \emph{How to put four elephants in a Citro{\"e}n 2CV ?} Answer: \emph{two in the front and two in the back.} 

 It seems to us that there is often a latent mishmash  between inconsistency  and incongruity. We believe that incongruity is an important ingredient which by its comic nature facilitates laughter, and thus adds to the surprise of the contrast (or even inconsistency) of the punchline with the expectation created by the context. Incongruity\footnote{Incongruity is lacking congruity which comes directly from late Latin \emph{congruus} "suitable, agreeing," \cite{incongruous}} usually involves the \emph{violation of norms} which can provoke a comic effect (insofar as the violation is not too shocking for the receiver). Hence incongruity also relates to inconsistency with respect to particular pieces of knowledge considered as "norms" (such are rules of good behavior) which are disregarded in the joke. Incongruity could be viewed as a spice, chosen in a more or less appropriate way, which is added in greater or lesser quantity to a dish according to the guests to whom it is served. Thus, for certain audiences, certain jokes will be too salty, or too bland.

We now recall several AI notions that are at work in our modeling of joke understanding, including our previous proposal, see \cite{DuPr2020}.

\section{Background on belief revision and nonmonotonic reasoning}\label{bg}

 With humor, there is a manipulation of what listeners are meant to believe in order to make them fall into a trap \cite{Raccah2015}. In this section, we simply propose to translate this idea in terms of belief revision as also suggested by  \cite{Ritchie2002} who introduced a logical language to formalize the interpretation of jokes. Ritchie started from the principle that ``the punchline forces to reinterpret the context previously set up''. His language includes an accommodation operator (used to integrate non-conflicting knowledge) and a belief change operator to take into account the punchline, this operator is inspired by revision in the sense of  \cite{Gardenfors1988}. 

Moreover the comprehension of a joke by a listener involves  modeling her or his commonsense knowledge with which she or he interprets the story told. This is a matter of nonmonotonic reasoning. 
This background section first provides refreshers on belief revision, and then on nonmonotonic reasoning, before interfacing the two settings in the framework of possibility theory. A brief review on the modeling of surprises is also added.

\subsection{Belief revision refresher}\label{revision}

We temporarily abandon the theme of laughter to recall the notion of belief revision which is central in our analysis of the understanding of jokes.
%
\cite{AGM1985} (AGM) have introduced the concept of belief revision.
Belief revision is the process of determining what remains of old beliefs after the arrival of new information. The beliefs are represented by expressions in a formal language. 
\emph{Revision} consists in adding information while preserving consistency. 
This operation is necessary since inconsistency leads to a useless state of belief. The main contribution of the AGM article is the definition of a set of postulates that must be satisfied by any \emph{rational} revision operator. As emphasized by   ~\cite{Gardenfors1988} - see also \cite{Sombe1994} - these postulates are based on three principles:
 \begin{itemize}
 \item 
 a principle of consistency (the result must be consistent),
 \item 
a principle of minimal change (the initial beliefs must be modified as little as possible),
 \item 
a principle of priority to the new information (the new information must hold after the revision).
 \end{itemize}

We consider a propositional language $\L$, where propositions are denoted by Greek symbols in lower case. The symbols $\bot$, $\vee$, $\wedge$, $\neg$, $\to$, $\equiv$, $\models$ denote respectively the contradiction,   the logical connectors ``or'', ``and'', ``not'', material implication, logical equivalence, satisfiability. Let $\Omega$ denote the set of interpretations induced by $\L$, we will often use $\omega$ for naming a particular interpretation in $\Omega$, each interpretation will be described by the list of literals satisfied by it , e.g. $\omega=a~\neg b~c$ is an interpretation that associates  the truth value True to $a$ and $c$ and False to $b$. 
$Mod(A)\subseteq \Omega$ is the set of interpretations satisfying the set of propositional formula $A\subseteq \L$ ($Mod(A)=\{\omega\in \Omega| \omega \models \bigwedge_{\varphi\in A} \varphi\}$), the same notation is used to represent the set of models of a formula $Mod(\varphi)=\{\omega\in \Omega|\omega \models \varphi\}$. 
In what follows, we use the symbol $\circ$ for representing an operator of \emph{belief revision} in the sense of \cite{KaMe1991}. 

We recall below the set of postulates stated by  \cite{KaMe1991} which are equivalent to those of AGM, but allow more easily to link a revision operator to a distance relation between interpretations.
More formally, a KM revision operator associates to a formula\footnote{In the initial approach of AGM, a deductively closed set $K$ of formulas was considered to represent the initial knowledge, KM showed that one could state the postulates on a formula $\kappa$ whose set of consequences forms $K$.} $\kappa$ and to a formula $\varphi$ (representing the new information), another formula denoted by $\kappa \circ \varphi$. To be considered as ``rational'' the operator $\circ$ must satisfy the KM postulates:

\textbf{(KM1):\quad} $\kappa \circ \varphi \models\varphi$

\textbf{(KM2):\quad } If $\kappa \wedge \varphi$ satisfiable, then $\kappa \circ \varphi\equiv \kappa \wedge \varphi$

\textbf{(KM3):\quad } If $\varphi$ is satisfiable, then $\kappa \circ \varphi$ is also satisfiable

\textbf{(KM4):\quad } If $\kappa_1\equiv \kappa_2$ and $\varphi_1\equiv \varphi_2$ then $\kappa_1\circ \varphi_1\equiv \kappa_2\circ \varphi_2$

\textbf{(KM5):\quad } $(\kappa\circ \varphi)\wedge \psi \models \kappa\circ(\varphi\wedge \psi)$

\textbf{(KM6):\quad } If $(\kappa\circ \varphi)\wedge \psi$ is satisfiable then $\kappa\circ (\varphi \wedge\psi) \models (\kappa \circ \varphi)\wedge \psi$


The first four postulates are the basic postulates of revision.
The last two agree with the principle of minimization of change, according to \cite{KaMe1991}; similar postulates had been introduced before as a basis for rational choice functions, see \cite{Sen1971}.
More precisely, \textbf{(KM1)} imposes that the new information must hold after the revision. \textbf{(KM2)} dictates that when the new information is not inconsistent with the original beliefs then the revision is a simple expansion. \textbf{(KM3)} expresses that if the new information is consistent then the revised set of beliefs is consistent. \textbf{(KM4)} expresses that a revision operator is syntax independent. 
According to \cite{KaMe1991}, ``rule \textbf{(KM5)} says that our notion of closeness is well-behaved in the sense that if we pick any interpretation $\omega$ which is closest to $Mod(\kappa)$ in a certain set, namely $Mod(\varphi)$, and $\omega$ also belongs to a smaller set, $Mod(\varphi\wedge \psi)$, then $\omega$ must also be closest to $Mod(\kappa)$ within the smaller set $Mod(\varphi\wedge\psi)$.''
Moreover \cite{KaMe1991} explained that ``a violation of rule \textbf{(KM6)} would imply that an interpretation $\omega$ may be closer to the knowledge base
than $\omega'$ within a certain set, while $\omega'$ is closer than $\omega$ within some other set''\footnote{\cite{KaMe1991} detailed this argument in the following way: ``consider a model $\omega$ of $\kappa \circ (\varphi \wedge \psi)$, that is, a model of $\varphi\wedge \psi$ that is closest to $Mod(\kappa)$ within the set $Mod(\varphi\wedge \psi)$. Suppose $\omega$ is not a model of $(\kappa \circ \varphi)\wedge \psi$. The precondition of \textbf{(KM6)} says that there is some interpretation $\omega'$ that is a model of $\kappa \circ \varphi$ and also of $\psi$. That is, $\omega'$ is a model of $\varphi \wedge \psi$ that is closest to $Mod(\kappa)$ within the set $Mod(\varphi)$. But then $\omega'$ is closer to $Mod(\kappa)$ within the set $Mod(\varphi)$ than $\omega$, while $\omega$ is closer to $Mod(\kappa)$ than $\omega'$ within the set $Mod(\varphi\wedge \psi)$.''}.


Katsuno and Mendelzon stated the following representation theorem for expressing a revision in terms of proximity of models:
\begin{theo}[\cite{KaMe1991}]\label{theoKM}
$\circ$ satisfies the postulates (KM1) - (KM6) if and only if there exists a function faithfully assigning to each epistemic state $\kappa$ a total pre-order $\preceq_\kappa$ such that:
$$Mod (\kappa \circ \varphi) = \min(Mod (\varphi), \preceq_\kappa)$$
\end{theo}

The assignment is faithful when it associates to any formula $\kappa$ a pre-order $\preceq_\kappa$ which strictly privileges the interpretations satisfying $\kappa$ to the others\footnote{More formally, the assignment is faithful if for any formula $\kappa$, we have:
\begin{itemize}
\item[1)] If  $\omega\models \kappa \mbox{ and } \omega'\models \kappa \mbox{ then }\omega =_\kappa \omega'$.
\item[2)] If  $\omega\models \kappa \mbox{ and } \omega'\not\models \kappa \mbox{ then }\omega \prec_\kappa \omega'$.
\item[3)] If  $\kappa \equiv \kappa'$ then  $\preceq_\kappa = \preceq_{\kappa'}$.
\end{itemize}}.
We recall that a \emph{pre-order} is a reflexive and transitive relation; a pre-order $\preceq$ over a set $A$ is \emph{total} when for any pair $(a,b)$ of elements of $A$ either $a \preceq b$ or $b\preceq a$. 
Note that in the context of belief revision, the piece of information $\varphi$ is a new piece of knowledge about the world which is considered as static. When $\varphi$ is representing an evolution of the world and not an evolution of the knowledge about the world, then the operation is called \emph{update} \cite{Winslett1988,KaMe1991b}. 
Belief update has its own postulates and representation theorem, but this issue goes beyond the modeling needs of this paper,  since in jokes, the aim of the teller is to manipulates the beliefs of the listener.

As noticed above a revision operator $\circ$ is defined by a family of pre-orders $\preceq_K$, one pre-order for each initial formula $K$, however in the following we only require to dispose of \emph{one} initial formula characterizing the state of knowledge of the listener before hearing the joke. This is why we only need one pre-order associated to $K$, namely the listener can be represented by the pair $(K, \preceq_K)$. Still in the definitions of next subsection, we shall  use the notation $(K,\circ)$ and benefit of concise writing such as $K \circ \varphi$. As explained in the next subsection, when the knowledge base of the listener includes default rules, it gives birth to a pre-order on interpretations. In other words, a set of default rules is an implicit  way to encode $(K,\preceq_K)$.


\subsection{Nonmonotonic reasoning and possibility theory refresher}\label{nmr}
As recalled by the above theorem, the semantics of belief revision is a matter of pre-orders which can be also expressed in terms of sphere systems (\cite{Grove1988})  and epistemic entrenchment (\cite{Gardenfors1988}), two notions which have their exact counterpart in possibility theory (in the form of qualitative possibility distributions and necessity relations, see \cite{DuPr1991}).  A possibility distribution is qualitative in the sense that it only plays the role of an ordinal ranking function on the set of possible interpretations (unlike a probability distribution). 
Possibility distributions are used here to implement the  revision operator $\circ$ because what is required is simply an order of plausibility as in sphere systems and no numerical value is either required or available. However, the approach presented in this paper is valid for any type of belief revision operator satisfying the \cite{AGM1985} or \cite{KaMe1991} postulates.

\cite{Gardenfors1990} has pointed out that belief revision and nonmonotonic reasoning can be seen as the two sides of the same coin. Indeed, a pre-order is naturally induced by default rules which express generic knowledge. 
The interest of default rules is the capability of reasoning in presence of incomplete information. In such a setting non-monotonic reasoning takes place for handling exceptions. 
In our encoding of common sense knowledge, we use $\leadsto$ to represent a default rule: $\alpha\leadsto \beta$, with $\alpha,\beta\in \L$; which means that when $\alpha$ is true, it is more plausible that $\beta$ is true than false. 

A procedure, called System Z, that stratifies a set of default rules (according to their ability to tolerate each other) was first proposed by \cite{Pearl1990} (in agreement with an infinitesimal probability setting). This stratification provides ways to obtain pre-orders on interpretations as described in \cite{BCDLP1993}, among them is the so-called \emph{best-out ordering} that prefers interpretations where the priority level of the highest priority formula violated is as small as possible. Another well-known ordering on interpretations based on stratified sets of formulas is the so-called \emph{lexicographic ordering} that labels each interpretation with a tuple listing the number of violated formulas per strata and compare the obtained tuples by lexicographic order, see \cite{BCDLP1993}.

The best-out ordering on interpretations corresponds to the one obtained by \cite{BDP1992} in a possibility theory setting. Moreover in \cite{BDP1992} a stratification of the set of default rules is also obtained, which is the same as the one obtained by System Z (\cite{Pearl1990}); see \cite{BDP1997}. 
More precisely, a default rule is translated in the framework of possibility theory (\cite{DuPr1988}) by the constraint $N(\beta|\alpha)>0$.  
Because of the duality between a necessity measure $N$ and its associated possibility measure $\Pi$, i.e., $N(\alpha)=1-\Pi(\neg \alpha)$, the constraint $N(\beta|\alpha)>0$ is written in an equivalent way as $\Pi(\alpha\wedge \beta)> \Pi(\alpha\wedge \neg \beta)$ which expresses that, in the context $\alpha$, having $\beta$ true is the normal situation (since strictly more possible than $\beta$ false). This writing assumes the existence of a plausibility relation on the interpretations, which is represented by a possibility distribution $\pi$, i.e., a function from the set of interpretations $\Omega$ of $\L$ to a linearly ordered scale bounded by 0 and 1, where 0 represents the impossible and 1 the completely possible. 
Given a set $\Delta$ of $n$ default rules $\alpha_i\leadsto \beta_i$, encoded as a collection of constraints $\Pi(\alpha_i\wedge \beta_i)> \Pi(\alpha_i\wedge \neg \beta_i)$, a unique possibility distribution $\pi_\Delta$ over interpretations is obtained by selecting the least restrictive\footnote{In terms of letting the possibility degrees as large as possible with respect to the constraints.} possibility distribution satisfying all the constraints (where $\Pi(\varphi)=\max_{\omega\models \varphi} \pi_\Delta(\omega)$). 
This is how the set of defaults $\Delta$ is associated with a pair $(K,\preceq_K)$ such that $K=\bigcup_i \neg \alpha_i \vee \beta_i$ and $\preceq_K=\preceq_\Delta$ where $\preceq_\Delta=\{(\omega,\omega')|\pi_\Delta(\omega)\geq \pi_\Delta(\omega')\}$ which is faithful with respect to $K$ by construction (as required in the Theorem \ref{theoKM}). Note that the formulas in $K$ can be ranked by using their necessity levels $N(\neg \alpha_i\vee \beta_i)$ induced from $\pi_\Delta$ by $N(\varphi)=\min_{\omega \not\models \varphi} (1- \pi_\Delta(\omega))$.

Now given a set of default rules $\Delta$ representing the initial knowledge, the arrival of a new piece of information $\varphi$ is integrated by a revision process (hidden behind the notation ``$K$ $\circ$ $\varphi$'')  where $(K,\preceq_K)$ are defined as above, based on $\Delta$. Note that this revision is equivalent to the conditioning of $\pi_\Delta$ by $\alpha$ (i.e., $\pi_\Delta(\omega|\varphi)=\pi_\Delta(\omega)$ if $\omega \models \varphi$, and $0$ otherwise), yielding a revision operator (\cite{BDPW2002}) satisfying the revision postulates. This operator is similar to that of \cite{Lehmann1995}.

As noticed in subsection \ref{revision}, the listener who could be equated with a pair $(K, \circ)$ can be more simply identified with a pair $(K,\preceq_K)$, $K$ being given. Such an ordering $\preceq_K$ is naturally available when the knowledge is expressed by a set of formulas including default rules (and maybe strict rules). 
Understanding a joke being a matter of default reasoning, as explained in the introduction of Section \ref{bg}, it leads to use default rules to represent the knowledge of the listener.   
This way of encoding the agent's \emph{non-monotonic} knowledge does not affect the generality of our approach. It is simply a compact and explicit way of specifying the plausibility relations between interpretations (which according to \cite{Gardenfors1990} underlie any revision operator).

\subsection{From default reasoning to belief revision}\label{defaultrevision}
From now on, the listener's knowledge can also be expressed by a knowledge base $\Sigma$ made of two parts: $\Sigma=(P, \Delta)$ which is the pair of a \emph{consistent} set $P$ of propositional formulas and a \emph{consistent}\footnote{The consistency of a default base is equivalent to the consistency of the set of constraints of the form $\Pi(\alpha_i\wedge \beta_i)>\Pi(\alpha_i\wedge \neg \beta_i)$ (\cite{BDP1997}).} set $\Delta$ of default rules of the form $\alpha_i \leadsto \beta_i$, with $\alpha_i, \beta_i\in \L$.  
Indeed, as it is the case in classical logic, an inconsistent knowledge base leads to an impossibility of rational reasoning, thus here the impossibility of understanding the joke. In this paper, $\Sigma$ plays the role of both $K$ and $\circ$ since $\Delta$ allows to induce a pre-order on the interpretations. As seen in subsection \ref{nmr}, $\Delta$ and $(K,\circ)$ are the two sides of the same coin since $Mod(K \circ \varphi)=\min(Mod(\varphi),\preceq_\Delta)$. With the introduction of the set of strict rules $P$ it is necessary to use a revision under the integrity constraint $P$ in the sense of \cite{KaMe1991} defined by $K \circ^P \varphi=K\circ (P\cup\{\varphi\})$ which yields $Mod(K\circ^P \varphi)=\min(Mod(P\cup\{\varphi\}),\preceq_\Delta)$ and which amounts to be equivalent with $\min(Mod(\varphi),\preceq_\Sigma)$ where $\preceq_\Sigma$ is the pre-order obtained from $\preceq_\Delta$ restricted to the models of $P$. 
More precisely, from the set $\Sigma$, we can compute  a possibility distribution $\pi_\Sigma$ over the interpretations that satisfy the constraints induced by $\Delta$ and assigns 0 to each interpretation that does not satisfy $P$ (\cite{BDP1997}). 

The aim of the following example is to show how the pair $(K,\preceq_K)$ can be obtained from a knowledge base $\Sigma=(P, \Delta)$. This is done in terms of an apparently abstract example, whose data corresponds exactly to the situation of the joke Example \ref{femmechezmedecin}. 

\begin{ex}\label{ktcr} Let us consider the following knowledge base $\Sigma$ representing the knowledge of a listener, we are going to show how this listener can rank order the interpretations based on $\Sigma$.\\
$\Sigma=\begin{array}{|ll}\small
P & \neg k \to \neg c\\
\hline
&t \leadsto c\\ 
\Delta &c \leadsto \neg r\\ 
&\neg c \leadsto r\\
\end{array}$

In this example the language is based on four propositional variables: $k$, $t$, $c$, $r$ yielding 16 interpretations that we denote by $\omega_1$, $\ldots$, $\omega_{16}$ as follows:
\setlength{\arraycolsep}{0pt}
$$\begin{array}{cccccp{1cm}ccccc}
~\omega_1~=~ & k & t & c & r & &~\omega_9~=~ & \neg k & t & c & r\\ 
~\omega_2~=~ & k & t & c & \neg r & &~\omega_{10}~=~ & \neg k & t & c & \neg r\\ 
~\omega_3~=~ & k & t & \neg c & r & &~\omega_{11}~=~ & \neg k & t & \neg c & r\\ 
~\omega_4~=~ & k & t & \neg c & \neg r & &~\omega_{12}~=~ & \neg k & t & \neg c & \neg r\\ 
~\omega_5~=~ & k & \neg t & c & r & &~\omega_{13}~=~& \neg k & \neg t & c & r\\ 
~\omega_6~=~ & k & \neg t & c & \neg r & &~\omega_{14}~=~ & \neg k & \neg t & c & \neg r\\ 
~\omega_7~=~ & k & \neg t & \neg c & r & &~\omega_{15}~=~ & \neg k & \neg t & \neg c & r\\ 
~\omega_8~=~ & k & \neg t & \neg c & \neg r & &~\omega_{16}~=~ & \neg k & \neg t & \neg c & \neg r\\
\end{array}$$

From $\Sigma$, applying the mechanism recalled in Section \ref{nmr}, 
\begin{enumerate}
\item $\Delta$ is transformed into the three constraints:\\ 
$\Pi(t\wedge c)> \Pi(t\wedge \neg c)$\\  
$\Pi(c\wedge \neg r)> \Pi(c\wedge r\}$\\ 
$\Pi(\neg c\wedge r)> \Pi(\neg c\wedge \neg r)$.\\
They are equivalent to:\\ $\max(\pi_\Delta(\omega_1),\pi_\Delta(\omega_2),\pi_\Delta(\omega_9),\pi_\Delta(\omega_{10}))>\max(\pi_\Delta(\omega_3),\pi_\Delta(\omega_4),\pi_\Delta(\omega_{11}),\pi_\Delta(\omega_{12}))$\\
$\max(\pi_\Delta(\omega_2),\pi_\Delta(\omega_6),\pi_\Delta(\omega_{10}),\pi_\Delta(\omega_{14}))>\max(\pi_\Delta(\omega_1),\pi_\Delta(\omega_5),\pi_\Delta(\omega_{9}),\pi_\Delta(\omega_{13}))$\\
$\max(\pi_\Delta(\omega_3),\pi_\Delta(\omega_7),\pi_\Delta(\omega_{11}),\pi_\Delta(\omega_{15}))>\max(\pi_\Delta(\omega_4),\pi_\Delta(\omega_8),\pi_\Delta(\omega_{12}),\pi_\Delta(\omega_{16}))$
\item The least restrictive possibility distribution satisfying the constraints is such that $\pi_\Delta(\omega_2)=\pi_\Delta(\omega_6)=\pi_\Delta(\omega_7)=\pi_\Delta(\omega_{10}) = \pi_\Delta(\omega_{14})=\pi_\Delta(\omega_{15}) > \pi_\Delta(\omega_1)=\pi_\Delta(\omega_3)=\pi_\Delta(\omega_4)=\pi_\Delta(\omega_{5}) = \pi_\Delta(\omega_{8})=\pi_\Delta(\omega_9)=\pi_\Delta(\omega_{11})=\pi_\Delta(\omega_{12})=\pi_\Delta(\omega_{13}) = \pi_\Delta(\omega_{16})$
\item We compute $\pi_\Sigma$ from $\pi_\Delta$ taking into account the constraint $\neg k\to \neg c$ of $P$, by making its counter-example at the lowest level:
$\pi_\Sigma(\omega_2)=\pi_\Sigma(\omega_6)=\pi_\Sigma(\omega_7)=\pi_\Sigma(\omega_{15}) > \pi_\Sigma(\omega_1)=\pi_\Sigma(\omega_3)=\pi_\Sigma(\omega_4)=\pi_\Sigma(\omega_{5}) = \pi_\Sigma(\omega_{8})=\pi_\Sigma(\omega_{11})=\pi_\Sigma(\omega_{12})= \pi_\Sigma(\omega_{16}) > \pi_\Sigma(\omega_9)=\pi_\Sigma(\omega_{10}) = \pi_\Sigma(\omega_{13}) = \pi_\Sigma(\omega_{14})=0$
\item This corresponds to the pair $(K,\preceq_K)$ where
$K=\{\neg t \vee c, \neg c\vee \neg r, c\vee r\}$
and $\preceq_K=\preceq^{bo}_\Sigma$ is such that\footnote{$\preceq^{bo}_\Sigma$ stands for best-out pre-ordering.} $\{\omega_2,\omega_6,\omega_7,\omega_{15}\} \preceq^{bo}_\Sigma \{\omega_1,\omega_3,\omega_4,\omega_{5},\omega_{8},\omega_{11},\omega_{12}, \omega_{16}\}$. The set $\{\omega_9,\omega_{10}, \omega_{13},\omega_{14}\}$ being excluded as counter-models of $P=k\vee \neg c$. 
  \item Another option is to provide the pair $(K,\preceq^{lex}_\Sigma)$ with a more refined ordering taking into account the number of violated formulas. The lexico pre-order (denoted $\preceq^{lex}_\Sigma$) can be obtained by first stratifying the formulas of $K$ as follows: for each default rule of $\Delta$ we compute the degree of necessity of its propositional counterpart: $\neg t \vee c$ is violated by $\{\omega_3, \omega_4,\omega_{11},\omega_{12}\}$, $\neg c\vee \neg r$ is violated by $\{\omega_1, \omega_5,\omega_{9},\omega_{13}\}$, $c\vee r$ is violated by $\{\omega_4, \omega_8,\omega_{12},\omega_{16}\}$ leading to assign to all these three formulas the same necessity, hence the same stratum. All formulas in $P$ are in a separated stratum with the highest priority (i.e., with a necessity degree equal to 1). It allows us to refine the previous pre-order $\preceq^{bo}_\Sigma$ on interpretations by counting the violated formulas in each stratum. For instance $\omega_{11}$ and $\omega_{12}$ where equivalent with respect to the best-out ordering $\preceq^{bo}_\Sigma$. Considering lexico-ordering $\preceq^{lex}_\Sigma$ we get a tuple $(0,1)$ for $\omega_{11}$ meaning that $\omega_{11}$ violates 0 formula in the stratum of highest priority and 1 in the other stratum (namely $\neg t\vee c$), we get $(0,2)$ for $\omega_{12}$ (namely violating $\neg t\vee c$ and $c \vee r$) yielding $\omega_{11}\preceq^{lex}_\Sigma \omega_{12}$. The lexico-ordering gives more precisely:
  $\{\omega_2,\omega_6,\omega_7,\omega_{15}\} \preceq^{lex}_\Sigma \{\omega_1,\omega_3,\omega_{5},\omega_{8},\omega_{11}, \omega_{16}\}\preceq^{lex}_\Sigma \{\omega_4,\omega_{12}\}$, the set $\{\omega_9,\omega_{10}, \omega_{13},\omega_{14}\}$ being again excluded as counter-models of $P$.
\end{enumerate}
\end{ex}

After showing how a commonsense knowledge base naturally induces a belief revision operator, we end this background section with a brief review of surprise modeling in the literature.

\subsection{From nonmonotonicity to surprises}\label{surprise}

Surprise arises from the occurrence of something considered practically impossible or unexpected. The first version of the possibility theory (\cite{Zad78}, \cite{DuPr1988}) proposed by the English economist \cite{Shackle1961} was based on the notion of the degree of surprise associated with an event, which was in fact a degree of impossibility of this event, calculated from a possibility distribution reflecting the uncertain knowledge about the state of the world considered. A fact $A$ is then all the more surprising that it is less coherent with what is imagined possible. Formally,
\begin{center}
 \texttt{surpriseDegree}$(A) = 1 - $\texttt{possibility}$(A)$. 
\end{center}

In AI, the question of surprises has been studied by  \cite{LoCa2007} who associate to surprise an incomprehension due to the difficulty to integrate the new perception, ``she cannot believe it''. The degree of surprise or astonishment is then associated with the probability of this new perception formalized in the \emph{logic of probabilistic quantified beliefs} by  \cite{FaHa1994}. For these authors, the notion of surprise is first defined by a form of inconsistency with explicit expectations or with a priori knowledge. Then, once the surprise has arrived, the agent must revise his knowledge. \\

We now are in position to present our formalization of joke understanding in terms of belief revision.

\section{Formalization of surprising and revealing statements}\label{formalisation}
One of the first models proposed in linguistics for the analysis of jokes is that of \cite{Suls1972}. 
It is a two-phase model: the punchline which is unexpected,
followed by a resolution (thought in terms of problem solving) which restores consistency (see also \cite{Shultz1976}). 
\cite{Ritchie1999,Ritchie2002} 
works on riddles such as ``Why do birds fly south in winter? It’s too far to walk''. He identifies the presence of four entities in such riddles: a simple (most obvious) interpretation of the context, a hidden meaning of this context, a meaning of the punchline and an interpretation formed with integrating the punchline with the hidden meaning of the context.
 Let us note that in other riddles the unexpectedness could rely on double-entente:
``Postmaster: Here’s your five-cent stamp.
Shopper: Do I have to stick it on myself?
Postmaster: Nope. On the envelope.''

In the approach by \cite{DuPr2020}, belief change is used to define the surprise and thus to model its integration in the listener's mind like Lorini and Castelfranchi did. Indeed the punchline of the joke seen as a new piece of information triggers a revision of the beliefs whose result is surprising in the sense that it is inconsistent with what was initially believed. 

Thus, we  formalized the listener manipulation as a two-phase process (\cite{DuPr2020,DuPr2022}):

\begin{itemize}
\item The first phase is a revision of the listener's beliefs with incomplete information intended to suggest a conclusion (which will turn out to be false). 
\item The second phase corresponds to the arrival of the punchline, which on the one hand surprises because it is incompatible with the previous provisional conclusion, which is translated in our model by an inconsistency. Moreover, the punchline is revealing, which we express by the fact that it logically explains the initial information.
\end{itemize}

Moreover, in this study we limit ourselves to the handling of the reasoning carried out by an agent listening to a joke. We leave aside the reasoning that the speaker may have done to develop the joke by adapting it to what he knows about the audience, not to mention the evaluation of whether or not to tell the joke in the situation.


\subsection{Formalization of surprises in jokes}
A joke is a progressive story. This means that when the listener receives $\alpha$ (the context of the joke) he or she is put in a situation of reasoning under incomplete information. This is indeed the teller's trick to lead the listener to wrong conclusions from $\alpha$ in order to finally surprise him or her with the punchline $\beta$. Indeed the teller manipulates the belief states of the listener. This is made possible because the listener is ready to use default reasoning to make sense of $\alpha$. Consider the extreme case of a fully psycho-rigid listener who would use only strict rules with all possible exceptions explicitly mentioned. He or she would not jump to any conclusion from $\alpha$ at all, and would wait for  $\beta$ for understanding what is going on. Thus this psycho rigid listener cannot be surprised (he or she is a monotonic thinker) and is not open to laughing.
\footnote{This is illustrated by the well-known story in which a mathematician and a physician are clearly reluctant to make a close world assumption in order to jump to a default conclusion in presence of incomplete information: A mathematician, a physicist and a biologist are on a train in Ireland. Through the window, they see a black sheep.
How interesting, exclaimed the biologist, in Ireland, sheep are black!
- You can't say that, replied the physicist. Certainly, there is at least one black sheep in Ireland.
- Come, come, continues the mathematician, the only thing that can be said is that there is at least one sheep, at least one side of which is black!
}

We now propose a formalization of jokes in propositional logic, i.e., we consider that the joke describes a situation $\alpha$ and that the punchline completes this description by an information $\beta$. The information $\alpha$ and $\beta$ are propositions of $\L$. 
We are interested in the case where the joke is heard by a listener whose knowledge base is a set of propositional formulas, here assimilated to its conjunction denoted $K$. Moreover the listener is also characterized by the way $\circ$ she revises her beliefs according to a pre-order $\preceq_K$ as explained in the previous section.

In the following, a joke is seen as a statement that is a pair (<context>, <punchline>) with particular characteristics. We define a statement as follows:

\begin{defi}
A \emph{statement} is a pair $(\alpha,\beta)$ formulas of $\L$.
\end{defi}

The cognitive situation induced by a statement is described by the knowledge base $K$ revised successively by the two elements of the statement. 
The punchline of a joke is considered \emph{surprising} if the result of the revision of $K$ by the description of the initial situation is contradictory to what we get after a revision by the punchline. Let us mention that in order to be surprise the listen should first be able to make an expectation from the context she or he heard. This is why the first condition to be surprised is that $K \circ \alpha$ is consistent.

\begin{defi}\label{surprenant}
A statement $(\alpha,\beta)$ is \emph{surprising} for a listener associated with $(K,\circ)$ if $$(K \circ \alpha) \mbox{ is consistent}\footnote{This condition is automatically ensured by the usual revision postulates (see \ref{revision}) as soon as $\alpha$ is consistent.}\mbox{ and }(K \circ \alpha)\wedge (K \circ (\alpha \wedge \beta))\models \bot$$
\end{defi}
Note that the above definition excludes the possibility that the listener already knows the joke. Indeed, in the latter case, $\alpha$ and $\beta$ would already be deducible from $K$, prohibiting any surprise. Besides, we do not require that $K\circ(\alpha \wedge \beta)$ be consistent, enabling nonsensical punchlines in Definition \ref{surprenant}\footnote{However, nonsensical punchlines will not be revealing in the sense of the next definition \ref{definevitable} due to postulates (KM5) and (KM6).}.

\begin{exs}{ktcr}
Now given the piece of information $t$ under the integrity constraints $P$ the listener equated with $(K,\preceq_K)$ will revise $K$ by $t$ by computing $Mod(K\circ^P t)=\min(Mod(t),\preceq_K)$ and obtain that the only model of the world is $\omega_2$.
Assuming that the new information is $t\wedge r$, three models correspond to $K\circ^P (t\wedge r)$ namely $\omega_1$, $\omega_3$ and $\omega_{11}$ where $\omega_2$ no longer appear, acknowledging that $(t,t\wedge r)$ is surprising for the listener.
\end{exs}

\subsection{Potentially funny jokes}
To understand the joke, its logic must seem implacable once the punchline is revealed, so the punchline is both admissible ($K \circ \beta$ consistent) and explains the situation. In other words, if one had known $\beta$ from the beginning, it would have explained $\alpha$.  This can be translated as follows:

\stepcounter{footnote}
\begin{defi}\label{definevitable}
Given a statement $(\alpha,\beta)$, its punchline $\beta$ is \emph{revealing} for $(K,\circ)$ if $$(K \circ \beta) \mbox{ is consistent}\footnotemark[\value{footnote}] 
\mbox{ and }K\circ\beta\models \alpha$$
\end{defi}
\footnotetext{This condition is automatically ensured by the revision postulates as soon as $\beta$ is consistent, see Section \ref{revision}.} 

Laughter can be viewed as a way of relieving the tension of cognitive dissonance caused by the inconsistency between what one expected to hear and the punchline of the story. This relief can only take place once the story is understood, with the punchline playing the role of a revelation.

\begin{defi}
A statement is \emph{potentially funny} for a listener associated with $(K,\circ)$ if it is surprising and its punchline is revealing for this listener.
\end{defi}

Figure \ref{potfun} pictures a representation of a pre-order $\preceq_K$ where each layer is a set of interpretations at the same plausibility level (the most plausible layer being at the bottom of the figure).  Models of $\alpha$ are in the dashed circle, and those of $\beta$ are in the dotted circle. The models of $K \circ \alpha$ and $K\circ \beta$ (which here are equal to $K\circ (\alpha \wedge \beta)$) are materialized by a thick dashed segment and a thick dotted segment respectively. One can check that the two segments are distinct one being strictly less plausible than the other making clear that $Mod(K\circ \alpha)\cap Mod(K\circ(\alpha \wedge \beta))=\emptyset$, i.e., $(\alpha,\beta)$ is surprising for the listener $(K,\circ)$. Moreover $K\circ\beta\models \alpha$ hence the punchline $\beta$ is revealing for our listener.

\begin{figure}[h]
\begin{center}
\tikzset{>={Latex[scale=1.1]}}
\begin{tikzpicture}
\draw[line width=1pt] (0,4) -- (5,4);
        \draw[line width=1pt,name path=line3] (0,3) -- (5,3);
        \draw[line width=1pt] (0,2) -- (5,2);
        \draw[line width=1pt,name path=line1] (0,1) -- (5,1);
        \draw[line width=2pt] (0,0) -- (5,0) ;
        \node[anchor=west] (labelk) at (6,0) {$Mod(K)$};
\node[blue,ellipse, line width=2pt,draw, dashed, minimum height=3cm, minimum width=2.5cm,name path=lalpha] (alpha) at (2.5,2) {};
         \node[blue,xshift=.9cm,yshift=-.5cm] at (alpha.north east) {$Mod(\alpha)$};
         \node[red,ellipse, draw, line width=2pt,dotted, minimum height=2cm,minimum width=2cm,name path=lbeta] (beta) at (1.7,3.5) {};
         \node[red,xshift=.5cm,yshift=.3cm] at (beta.north east) {$Mod(\beta)$};
\draw[blue,name intersections={of=line1 and lalpha},line width=3pt, dashed]
    (intersection-1) --    (intersection-2) ;
\draw[red,name intersections={of=line3 and lbeta, by={I1,I2}}, name intersections= {of=line3 and lalpha, by={I3,I4}}, line width=3pt, dotted]
    (I4) -- (I2) ;
    \node[anchor=west] (labelkob) at (6,3.7) {$Mod(K\circ \beta)$} ;
    \node[anchor=west] (labelkoa) at (6,1.5) {$Mod(K\circ \alpha)$} ;         
\draw[->,densely dotted] (labelkob) -- (2.4,3.7) -- (2,3.1);
\draw[->,densely dotted] (labelkoa) -- (2.9,1.5) -- (2.5,1.1);
\draw[->,densely dotted] (labelk) -- (5.2,0);

\draw[->] (-1,-.5)node[below]{$\preceq_K$} -- (-1,5); 
\end{tikzpicture}
\end{center}
\caption{A typical situation in which the statement $(\alpha,\beta)$ is potentially funny}\label{potfun}
\end{figure}
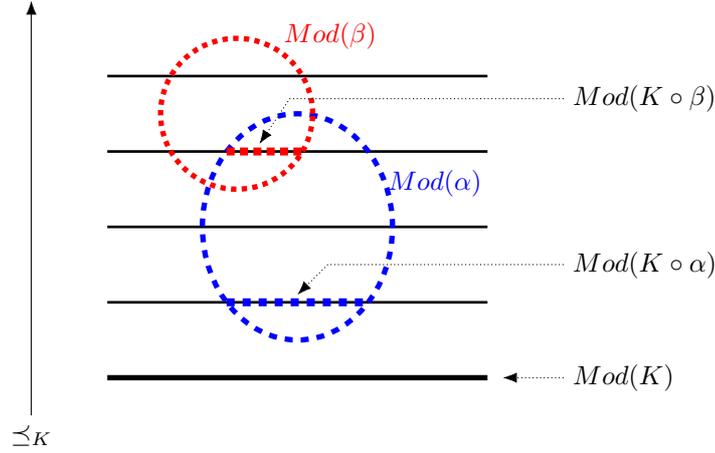

This definition of a potentially funny statement is directly inspired by \cite{Raccah2015}, except that we tried to define the revealing character of the punchline rather than its inevitability (which could be understood as $K \circ \alpha\models \beta$, denying the surprising character of $\beta$). 
Note that the definition relies on both $K$ and $\circ$, which may differ from one listener to another, which is consistent with the fact that a joke may not be found funny by everyone.

The following proposition restates the definition of a potentially funny joke in a nonmonotonic setting: namely it shows how a listener equated with a pair $(K,\circ)$ induced from her or his defeasible knowledge base $\Sigma(P,\Delta)$, can find a joke surprising and revealing. Note that this definition is done for statements $(\alpha,\beta)$ consistent with $P$, in agreement with the conditions $K \circ \alpha$ and $K\circ \beta$ consistent of Definitions \ref{surprise} and \ref{definevitable}.

\begin{prop}\label{transform} Let $\Sigma=(P,\Delta)$ be a knowledge base where $P$ is a consistent set of propositional formulas and $\Delta$ is a consistent set of default rules and let $\preceq^{bo}_\Delta$ be the total best-out pre-order on the set of interpretations $\Omega$ uniquely defined from $\Delta$. Let $(K,\circ)$ be such that $K=\{\neg \alpha_i \vee \beta_i|\alpha_i\leadsto \beta_i\in \Delta\}$ and $\circ$ be an operator defined from the total pre-order $\preceq^{bo}_\Delta$ for all $\varphi\in \L$ by:  $Mod(K\circ \varphi)=\min(Mod(\varphi),\preceq^{bo}_\Delta)$.

Let $\alpha,\beta\in \L$ such that $\{\alpha,\beta\}$ is consistent with $P$,
\begin{itemize}
\item The statement $(\alpha,\beta)$ is \emph{surprising} for a listener, equated with $(K,\circ)$, if and only if  $ \min(Mod(\alpha),\preceq^{bo}_\Delta)\cap \min(Mod(\alpha\wedge \beta),\preceq^{bo}_\Delta)=\emptyset$ 

\item The statement  $(\alpha,\beta)$ is \emph{revealing} for $(K,\circ)$ 
\quad if and only if \quad $\beta \entails^{bo}_\Delta \alpha$ 
\end{itemize}
where $\entails^{bo}_\Delta$ is the non-monotonic inference relation defined by $\varphi \entails^{bo}_\Delta \psi$ if and only if for each $\omega \in \min(Mod(\varphi),\preceq^{bo}_\Delta)$, $\omega \models \psi$. 

The same proposition holds for the lexicographic ordering, i.e., if we replace ``bo'' by ``lex''. 
\end{prop}

\begin{proof} Let us first show that $\circ$ is a revision operator satisfying KM postulates: this is due to the fact that any interpretation that does not violate any default of $\Delta$ is in the set of the interpretations that are the most preferred ones according to $\preceq^{bo}_\Delta$. Now, each interpretation that satisfies $K=\bigwedge_{\alpha_i\leadsto \beta_i\in \Delta} \neg \alpha_i\vee \beta_i$ does not violate any default, which means that the function that assigns to $K$ the total pre-ordering $\preceq_K=\preceq^{bo}_\Delta$ is faithful, due to Theorem \ref{theoKM} we get that $\circ$ satisfies KM postulates.

Due to Definition \ref{surprenant}, a statement $(\alpha,\beta)$ is surprising for $(K,\circ)$ if $(K \circ \alpha)$ is consistent and $(K\circ \alpha) \wedge (K \circ (\alpha\wedge \beta))\models\bot$, due to the assumption that $\{\alpha,\beta\}$ is consistent with $P$, it means that $\alpha$ is consistent hence due to (KM3), $K\circ \alpha$ is consistent. Moreover by definition of $\circ$ based on $\Delta$, $Mod(K\circ \alpha)=\min(Mod(\alpha),\preceq^{bo}_\Delta)$ and $Mod(K\circ (\alpha\wedge \beta))=\min(Mod(\alpha\wedge\beta),\preceq^{bo}_\Delta)$. Hence the result.

Due to Definition \ref{definevitable}, a statement $(\alpha,\beta)$ is revealing for $(K,\circ)$ if $K \circ \beta$ is consistent and $K\circ \beta\models \alpha$. Again, due to the assumption that $\{\alpha,\beta\}$ is consistent with $P$, it holds that $\beta$ is consistent due to (KM3), $K\circ \beta$ is consistent. Moreover due to the definition of $\circ$ based on $\Delta$, $Mod(\!K\!\circ\! \beta)\!=\!\min(Mod(\beta),\leq^{bo}_\Delta)$. Hence the result.

This proof holds if we replace the best-out ordering with the lexicographic ordering, because the interpretations that do not violate any default are also the preferred ones for $\preceq^{lex}_\Delta$. 
\end{proof}

We are now in position to show that a psycho-rigid  listener (i.e., an agent that describes the world only with strict rules) cannot understand a potentially funny joke, because he cannot be surprised:
\begin{prop}[psycho-rigidity]
If $\Sigma=(P,\Delta)$ is a knowledge base where $P$ is a consistent set of propositional formulas and $\Delta$ is empty, then for any pair $(\alpha,\beta)$ such that $\{\alpha,\beta\}$ is consistent with $P$, the statement $(\alpha,\beta)$ is \emph{not potentially funny} for the listener $(K,\circ)$ built from $\Sigma$ as in the previous proposition.
\end{prop}

\begin{proof} If $\Delta$ is empty then $\preceq^{bo}_\Delta$ ranks every interpretation at the same priority level, hence since $\alpha$ and $\alpha\wedge\beta$ are consistent, they have common minimum models with respect to $\preceq^{bo}_\Delta$. Thus, the statement $(\alpha,\beta)$ is not surprising for $(K,\circ)$. This holds also with the $\preceq^{lex}_\Delta$ pre-ordering.
\end{proof}

\subsection{Joke examples and further propositions}\label{jokes}
We are now in a position to formalize a first example (from a French collection of "funny" stories by  \cite{Negre1970}) expressed in natural language that we have transposed into logic. Note that in this article, we ignore the conversion stage from natural language  to logic.

\begin{ex}\label{medecinmeme} 
A man has just been hit by a car. The driver gets out of the car and says: "You're lucky, we're right in front of a doctor's office. Yes! Except that the doctor is me!

Modeling:
\begin{center}
$\begin{array}{l}
\alpha=injured \wedge doctorNearby\\
\beta=injured \wedge doctorHimself\\
\end{array}$
\end{center}
We assume that the listener has the following knowledge:
$$\Sigma=\begin{array}{|l}\small
injured \wedge doctorNearby \leadsto treatedRapidly\\ 
injured \wedge \neg treatingDoctor \leadsto \neg treatedRapidly\\ 
doctorHimself \leadsto doctorNearby \\
injured \wedge doctorHimself \leadsto \neg treatingDoctor\\
\end{array}$$

As shown in Section \ref{defaultrevision}, from $\Sigma=(P,\Delta)$, $K$ and $\preceq_K$ can be built. 
The reader can check that we have the following pre-order on interpretations satisfying $\alpha$:\\ $\{R~\neg H~T~i~N\}\preceq_K \{R~H~\neg T~i~N,~ R~\neg H~\neg T~i~N,~ \neg R~ H~\neg T~i~N,~ \neg R~\neg H~T~i~N,~ \neg R~\neg H~\neg T~i~N\}$\\ $\preceq_K \{R~ H~T~i~N,~ \neg R~ H~T~i~N\}$ where $R$, $H$, $T$, $i$, $N$ are short for $treatedRapidly$, $doctorHimself$, $treatingDoctor$, $injured$, $doctorNearby$ respectively.

Due to postulate (KM1), $K \circ \alpha \models\alpha$. More precisely, the only model of $K \circ \alpha$ is $R~\neg H~T~i~N$ (it is the most plausible model of $\alpha$ according to $\preceq_K$).

\noindent The models of $K \circ(\alpha \wedge \beta)$ are $R~H~\neg T~i~N$ and $\neg R~ H~\neg T~i~N$. The punchline is surprising since $Mod(K\circ \alpha)\cap Mod(K\circ(\alpha\wedge \beta))= \emptyset$.

Moreover, the punchline explains $\alpha$: $K \circ \beta \models i \wedge N$ since $Mod(K\circ \beta)$ is the set of the best models of $\beta$ according to $\preceq_K$, which turn out to be here the same two models as the ones obtained for $K \circ(\alpha\wedge \beta)$. Indeed, $N=doctorNearby$ is the part of the joke that manipulated us by leading us to think that the wounded man was lucky.
\end{ex}

Note that Example \ref{medecinmeme} can also be seen as a cascading joke $(\varphi_1$, $\varphi_2$, $\varphi_3$, $\varphi_4)$ with $\varphi_1=injured$, $\varphi_2=lucky$, $\varphi_3=doctorNearby$, $\varphi_4=doctorHimself$. Indeed, saying ``You are very lucky'' after having noticed that the person is hurt can already provoke laughter. Then, the revelation of being ``in front of a doctor's office'' helps to understand. Finally, the second part of the joke corresponds to what has already been analyzed. Note that in the first part of the cascading joke, $\varphi_2$ is surprising but not revealing, it is $\varphi_3$ that explains $\varphi_2$. So for a cascading joke, the surprising and revealing effects are not necessarily simultaneous.

The following property shows the importance of narrative order in the joke.
\begin{prop}[Never give the punchline before the end!]
If $(\alpha,\beta)$ is a potentially funny statement for a listener knowing $K$, then $(\beta,\alpha)$ is not potentially funny for that listener.
\end{prop}
\begin{proof} 
Since $(\alpha,\beta)$ is potentially funny, then the punchline is revealing, that is, $K\circ\beta\not\models \bot$ and $K\circ\beta\models \alpha$. 
But that prevents the story $(\beta,\alpha)$ from being surprising, and therefore potentially funny. Indeed, using (KM5) and (KM6), $K\circ (\beta \wedge \alpha)\equiv (K\circ \beta)\cup \{\alpha\}$ so $K\circ \beta$ is consistent with $K\circ (\beta \wedge \alpha)$.
\end{proof}

The next property indicates the virtue of brevity.

\begin{prop}[the shortest jokes...]\label{propcourte}
If $(\alpha,\beta)$ is a potentially funny statement knowing $K$ then adding $\alpha'\in \L$, such that $\alpha\not\models\alpha'$, can make the statement $(\alpha\wedge\alpha',\beta)$ not potentially funny. 
\end{prop}
\begin{proof}
Indeed here is a counterexample: let $\Sigma=(P,\Delta)$ where $P=p\to b$,
$\Delta=\{b \leadsto f$,  $p\leadsto \neg f$, $b \wedge a \leadsto p\}$ a knowledge base expressing that: generally birds ($b$) fly ($f$), penguins ($p$) are birds, generally penguins ($p$) don't fly, most Antarctic birds ($a$) are penguins. Let us consider $(K,\circ)$ built as explained in Section \ref{defaultrevision}.
Let us imagine a story $(b,p)$, it is both surprising since $K\circ b\models b \wedge \neg p \wedge\neg a \wedge f$ and $K \circ (b\wedge p)\models b\wedge p\wedge \neg f$ and the ``punchline'' $p$ is revealing since $K\circ p \models b$. However, by adding information $\alpha'=a$ that the bird is from Antarctica we can create an unsurprising situation with respect to $\beta=p$, so $(b\wedge a, p)$ is no longer a surprising story. 
Another counterexample consists in adding information $\alpha'$ that is not related to $\beta$ then we can obtain a punchline that is not completely revealing: let us consider $\alpha'= t$, where $t$ means ``is called Tweety'' then $K \circ p\models b$ does not allow us to deduce that $K\circ p\models t$. 
\end{proof}

To convince ourselves of the validity of the Property \ref{propcourte}, we can consider the particular case $\alpha'=\beta$. This would correspond in the example \ref{medecinmeme} to specifying that the driver has just run over a doctor in front of the doctor's office. In this case, the punch line ``Yes! except that the doctor is me! is no longer funny (because we already know that). Nevertheless the new story can be found funny in a different way with $\alpha''=$ \emph{A man has just run over his own doctor in front of his office} and $\beta''=$ \emph{You're very lucky, we're right in front of my doctor's office}. Here, the statement is quite surprising, but the revelation can only come if we perceive the irony of the use of the word lucky (in place of unlucky).

\medskip

We can notice that \ref{propcourte} does not require that the punchline be as short as possible. The following definition refines the definition of ``funny'' into ``efficient'' by imposing minimality (in terms of interpretations).
\begin{defi}[conciseness of the punchline]
Let $\alpha,\beta,\beta'\in \L$ be individually consistent formulas and $\beta\not\equiv\beta'$, if $K \circ \beta \models \beta'$ then the statement $(\alpha,\beta)$ is said to be more \emph{efficient} than the statement $(\alpha,\beta')$ for listener knowing $K$.
\end{defi}

In other words, the punchline is said to be more efficient if it is more precise and therefore more specific. Here is another example from \cite{Negre1970}'s collection.
 
\setitemize{
   fullwidth,
   leftmargin=0cm,
   nolistsep
}
\begin{ex}\label{femmechezmedecin} A guy and a woman enter the doctor's office. The doctor turns to the lady and says  :\begin{itemize}
\item If you are ill, please undress... \\
But the girl has manners. She is reluctant. She looks at the guy from underneath. So the doctor repeats:
\item But come on, lady, get undressed! I am a doctor. There's no indecency, here!
Then she starts to fidget and suddenly bursts into tears. Disconcerted, the doctor asks the guy: 
\item What's wrong with your wife? Is she always so nervous?
\item I don't know. I just met her in your waiting room...
\end{itemize}
\medskip

Modeling:~\\
$\varphi_1=together$,\quad 
$\varphi_2=reluctant$, \quad
$\varphi_3=together \wedge \neg knowEachOther $\\
$\Sigma=\begin{array}{|ll}\small
P & \neg knowEachOther \to \neg couple\\
\hline
&together \leadsto couple\\ 
\Delta &couple \leadsto \neg reluctant\\ 
&\neg couple \leadsto reluctant\\
\end{array}$

As computed in Example \ref{ktcr} where $k$, $t$, $c$, $r$ are short for $knowEachOther$, $together$, $couple$, $reluctant$, the only model of $K \circ \varphi_1$ is $\omega_2$, the models of $K \circ (\varphi_1 \wedge \varphi_2)$ are $\omega_1$ and  $\omega_3$ hence $\varphi_2$ is surprising since $Mod(K\circ\varphi_1) \cap Mod(K\circ(\varphi_1\wedge\varphi_2)) =\emptyset$.
Moreover in the lexicographic version of the pre-ordering over interpretations ($\preceq^{lex}_\Sigma$) the punchline $\varphi_3$ explains $\varphi_2$: since the most plausible model of $\varphi_3$ according to $\preceq^{lex}_\Sigma$ is $\omega_{11}$ hence $K \circ \varphi_3 \models reluctant$. Note that this result is not obtained with the best-out pre-ordering $\preceq^{bo}_\Sigma$ which is not able to discriminate between $\omega_{11}$ and $\omega_{12}$ as noticed in Example \ref{ktcr}.

Suppose now that the guy replied $\varphi'_3=$``Sorry, we entered together but we are not married'' ($\varphi'_3= together \wedge \neg couple$). It is clear that $K\circ \varphi_3 \models \varphi'_3$. Hence the old punchline $\varphi_3$ is more efficient than the new punchline $\varphi'_3$. 
\end{ex}

Note that we might have considered iterated revision \cite{KoPP00} and abductive expansion \cite{Pagnucco96}. However, in order to capture the cognitive aspects of understanding a joke, the main idea is to select the best epistemic state rather than to maintain the order of preferences over all states as in iterated revision. Indeed, in the basic revision, the epistemic states can be radically changed. The surprise would not be the same if the auditor had in mind all the potential explanations. In this respect, a similar argument applies to abductive expansion, whose goal is to provide reasons for the new epistemic state. Furthermore, we have favored the simplicity of classical revision theory.

\subsection{Graduality in surprising and revealing jokes}
A lesson that we draw from Raccah's article concerns the intensity of funniness: according to \cite{Raccah2015}, ``the joke is all the funnier that the trap was unexpected and inevitable''. 
This leads us to propose to define two quantities to classify the jokes. 
This would allow us to build a partial order on the jokes addressed to the same listener reflecting their potential funniness. 
This definition assumes that the knowledge of the listener contains default rules, as previously assumed, it should have the form: $\Sigma=(P,\Delta)$ enabling us to compute a possibility distribution $\pi_\Sigma$. 
The surprise level is conform to the view recalled in subsection \ref{surprise}. The second term is the conditional necessity of $\alpha$ knowing $K\circ \beta$, reflecting the ineluctability of the revelation. 

\begin{defi}[gradual effects in jokes]
Given a base $\Sigma=(\Delta ,P)$, a statement $(\alpha,\beta)$ is associated to the possibilistic levels: \begin{itemize}
\item a surprise level: $\left\{\begin{array}{ll}
0 &\mbox{ if } N(\alpha)= N(\alpha \wedge \beta),\\ 
1- \Pi(\alpha\wedge\beta) & \mbox{ otherwise}\\
\end{array}\right.$
\item a revealing level: $N(\alpha|K\circ\beta)$
\end{itemize}
where $K\circ \varphi$ is defined from $\Sigma$ by building $(K,\preceq_K)$ and $N$ and $\Pi$ are the necessity and possibility measures associated with  $\pi_\Sigma$, as explained in subsection \ref{nmr}.
\end{defi}

Let's not forget that other elements can contribute to the funniness, as the huge literature on humor can testify. All these additional aspects are themselves gradual in nature: such as the cognitive effort necessary to understand the joke, the language level, the comic effects of the narrator, the mood of the listener, his capacity of inhibition of his emotions,  etc. This is  a topic for further research. 

Moreover the modeling of laughter is a multi-agent problem involving a teller and a listener, especially when we come to the handling of social notions like complicity between listener and teller, and acceptable incongruity: the teller should take into account what he knows of the listener's norms. We now further discuss incongruity.

\section{Integrating incongruity}\label{Incong}

This section deals with inconsistency (associated with incongruity) as a new feature in jokes; this time inconsistency occurs with respect to listener's (and teller's) knowledge and not because of the listener's ability to reason with incomplete information. We first discuss incongruity and then absurdity as a particular case, before
introducing the idea of ``acting as if’’ a law / norm could be ignored, which is related to incongruity as we shall see. 

\subsection{Introducing incongruity and absurdity}
As said in Section \ref{Ing}, incongruity is usually related to the violation of social norms (i.e., lack of conformity with current practices, unsuitable behavior). Let us start with examples.

\begin{ex}~
\begin{itemize}
\item To drink water from a finger bowl is an \emph{incongruous} behavior because it violates hygienic norms.

\item Wearing oversized shoes is \emph{incongruous} because shoes are made for easy walking. This is why seeing a clown with such shoes has a comic effect all the more as it causes his falling. 
\end{itemize}
\end{ex}

One can notice there  that the comic effect occurs because the joke is deliberately in disregard with an obvious universal knowledge. Acting as if a universal law was not mandatory is an incongruous behavior.

As illustrated by the next example, a suitable behavior when communicating with people is to respect common knowledge about the object / situation under discussion. When someone derogates to this social principle, an \emph{absurd} communication takes place.

\begin{ex}~
\begin{itemize}
\item Speaking about someone drinking the water from a finger bowl is a description of an \emph{incongruous} behavior but it is \emph{not absurd} with respect to communication: it is not impossible that someone does it.

\item Speaking about carrying an elephant in an ordinary car (with 4 seats) is an \emph{absurdity}, because as everybody knows this is not feasible (an elephant is too big and too heavy).
\end{itemize}
\end{ex}

So absurdity may be viewed as a special form of incongruity as far as it comes from the lack of respect of some communication principle. Note that incongruity, hence absurdity, may be felt as surprising inasmuch incongruity is unusual (since resulting from the violation of current practices).

One can observe that in many comic situations, some ``acting as if’’ is at work. This means that a person behaves without taking into account some obvious facts or pieces of knowledge about the world. Thus the clown strives to walk normally despite of his unsuitable shoes. In the four elephants riddle, one strives to reason normally despite of the feasible impossibility of carrying elephants in a car, as discussed now.

\begin{ex}\label{elephant}~\\
 $i,e,tt,h$ being propositional formulas representing respectively the fact that 4 elephants are in the 2CV car, four elephants are in presence, two elephants are in the front and two elephants are in the rear of the 2CV car, huge animals are in presence.\\
$\alpha$ :  $i \wedge e$\\
$\beta$: $tt \wedge e$\\
$\Sigma=(P, \Delta)$ with $P=\{R1,R2,R3,R4\}$ and $\Delta=\emptyset$ where \\
R1:$h \to \neg i$ \hfill{(Something huge cannot be inside a 2CV car)}  \\
R2: $e \to h$ \hfill{(Elephants are huge)}\\
R3: $h\to \neg tt$ \hfill{(Something huge cannot be put in the rear and in the front)}\\
R4: $tt \to i$ \hfill{(Putting in the rear and in the front implies putting inside)}\\

$P\cup\{\alpha\}$ is inconsistent, hence the notion of surprising for a listener (whose revision operator is induced by $\Sigma$) given in previous section (Proposition \ref{transform}) does not apply. 
Another kind of surprise is at stake: the violation of an obvious law (R2: elephants are huge), thus we are in presence of an absurdity in communication. In this example the absurdity appears already in $\alpha$: $P \models \neg \alpha$. 
But the revealing part is still there, when forgetting this obvious law (R2), namely, $\{\beta\}\cup (P\setminus \{R2\})$  $\models \alpha$ hence $\beta$ is revealing for an agent knowing $\Sigma$ and ignoring $R2$.\\
\end{ex}

We now give another example where incongruity comes from the violation of social norms.

\begin{ex}\label{fire} My house caught fire (fh), my neighbor welcomes me to his house. I set fire to his house (sFn). You wonder why. Because he told me to do what I do at home ((an): allowed to make himself at home at the neighbor's), and my home was in fire. \\
$\alpha$ = fh $\wedge$ sFn\\
$\beta$ = fh $\wedge$ an\\
$\Sigma=(P, \Delta)$ with $P=\{R2\}$ and $\Delta=\{R1,R3, R4\}$ where \\
R1: sFn  $\leadsto $ ir \hfill{(set fire at the neighbor's is generally an irrational act)}\\
R2: fh $\wedge$ dh $\to$ sFn  \hfill{(having a fire at home and doing like at home implies to setting fire at neighbor's)}\\
R3: an $\leadsto$ $\neg$ ir  \hfill{(when allowed to make yourself at home generally you do not perform irrational acts)}\\
R4: an $\leadsto$ dh \hfill{(if you are allowed to make yourself at home then generally you do like at home)}\\

The use of ``allowed’’ is a way to encode permission, here there is a violation of a social permission (rule $R3$), the storyteller is not allowed to destroy the house of his neighbor despite being allowed to make himself at home.  
Nevertheless there is a revealing part, because with $\beta$ and $(P ,\Delta \setminus \{R3\})$ we get $\alpha$ (we reason logically in disregard of $R3$ which encodes an expected basic attitude).
\end{ex}

Note that, in the previous example, the surprise comes from the fact that an obvious law is violated, the revealing side of the joke is obtained by ignoring this particular piece of knowledge. 
It leads us to define a new notion that was already implicitly present in the examples formalized in the previous section: the violation of a social norm, and the revealing with disregard of a formula. 

\subsection{Formal handling of incongruity}
Before defining formally incongruity, we introduce the nonmonotonic entailment operator used for representing the listener reasoning. In the following definition, a knowledge base $\Sigma=(P,\Delta)$  induces a pre-order $\preceq^m_\Sigma$ on the set of interpretations $\Omega$ ($m$ being either the best-out or the lexicographic method) where$\preceq^m_\Sigma$ is the restriction of $\preceq^m_\Delta$ to the models of $P$. We define $\entails^m_\Sigma$, the nonomonotonic inference relation based on $\Sigma$, as follows. In this definition, $P$ acts as an integrity constraint, hence entailment is possible only from formulas consistent with $P$.

\begin{defi}[nonmonotonic entailment under $\Sigma= (P, \Delta)$]
\label{snake}
$\varphi \entails^m_\Sigma \psi$ if 
\begin{itemize}
\item $P \cup \{\varphi\}$ is consistent and 
\item for all $\omega\in\min (Mod(P\cup\{\varphi\})\preceq^m_\Sigma)$, $\omega \models \psi$. 
\end{itemize}
\end{defi}

The  notion of disregard is now defined in the setting of a knowledge base $\Sigma$ composed with strict and default rules from which non-monotonic reasoning can be performed, as described in subsection \ref{defaultrevision}. For capturing this notion we would like to express 
that the reasoning should lead to conclude that a formula representing a norm is violated.
%
This means that the whole joke (the context of the punch line) contradicts at least one norm: this is why the following definition  characterizes the disregard of one formula $\rho$, considered as a norm by the listener. In presence of such an incongruity the revealing in the sense of Definition \ref{definevitable} is no more desirable since we would like to reason while ignoring the norm. This leads us to the following definition.

\begin{defi}[incongruity]  \label{disdef}
Given $\Sigma=(P,\Delta)$, let $P_N\subseteq P$ and $\Delta_N\subseteq \Delta$ be subsets of formulas considered as norms in $\Sigma$, let $\rho$ be a formula either in $P_N$ or in $\Delta_N$,
the statement ($\alpha,\beta$) is \emph{felt incongruous} by a listener represented by $\Sigma$ if:
\begin{enumerate}
\item  $\alpha \wedge \beta\entails^m_{\Sigma\setminus \{\rho\}}  \mbox{Not }  \rho$  \hfill{(violation of $\rho$)}
\item 
$\beta\entails^m_{\Sigma\setminus\{\rho\}} \alpha$ \quad \hfill{(revelation in disregard of $\rho$)} 
\end{enumerate}
where, \begin{tabular}[t]{lll}
when $\rho \in P$, & $\Sigma\setminus \{\rho\}=(P\setminus \{\rho\},\Delta)$  and & Not  $\rho=\neg \rho$;\\
when $\rho=\alpha_i\leadsto \beta_i\in \Delta$, &$\Sigma\setminus \{\rho\}=(P,\Delta \setminus \{\rho\})$ and & Not $\rho=\alpha_i\wedge \neg \beta_i$.
\end{tabular}
\end{defi}
Note that a particular case of violation (1) is when $\alpha\wedge \beta$ is inconsistent with $P$. 
Besides,  some jokes cannot be felt incongruous because they violate an integrity constraint that cannot be forgotten, as made clear by  the following proposition.

\begin{prop} \label{prop}Any formula $\rho\in P$ such that $P\setminus \{\rho\}\models \rho$ cannot be considered as violated by a listener $(P,\Delta)$.
\end{prop}
\begin{proof} By monotonicity of classical logical consequence, we would have $\{\alpha\wedge\beta\} \cup P\setminus\{\rho\} \models \rho$. Definition \ref{snake} would lead to have for all $\omega\in \min(Mod(\{\alpha\wedge \beta\} \cup P\setminus\{\rho\})$, $\omega \models \neg \rho$ which is contradictory. Hence the result. 
\end{proof}

 The situation is different when $\rho=\alpha_k\leadsto\beta_k \in \Delta$, it is possible that $\alpha\wedge \beta \entails^m_{\Sigma\setminus\{\rho\}} \alpha_i\wedge \neg \beta_i$ holds together with $\alpha_i \entails^m_{\Sigma\setminus\{\rho\}} \beta_i$ (i.e., the knowledge base $\Sigma\setminus\{\rho\}$ still entailing the default $\rho$). However, it is not a problem to allow such formulas to be disregarded because in presence of $\alpha\wedge\beta$ the nonmonotonic mechanism will favor the inference of $\alpha_i\wedge \neg \beta_i$, hence the violation of the norm $\rho$. 

In other words, incongruity is not so much a matter of \emph{removing} a norm from the listener knowledge base but rather a \emph{violation} of this norm by focusing on exceptional situations. For instance in the story about the four elephants, if we consider that all information is defeasible, then the violation of the rule that elephants are too huge to enter in a car is possible 
by considering exceptional elephants with no huge size. But such a consideration is by itself incongruous by violating the general law. 

We should notice that Definition \ref{disdef} is using a set difference operator in order to check if by removing $\rho$ in the knowledge base $\Sigma$, the statement $\alpha\wedge \beta$ violates $\rho$ and also to check if $\beta$ reveals $\alpha$ when removing $\rho$. The reader might wonder if $\Sigma\setminus \{\rho\}$ should be better replaced by a contraction\footnote{We recall that the contraction of a knowledge base by $\varphi$ is the operation that modifies it in order that it no longer entails $\varphi$. This operation was introduced in a propositional logic setting by  \cite{AGM1985}. More formally, $K$ contracted by $\rho$, denoted K$\dotminus$ $\rho$, can be computed as a revision using $(K \circ$ $\neg\rho$) $\wedge K$ according to Harper's identity (\cite{Gardenfors1988}). 
However, such a modeling does not seem very satisfactory from a cognitive point of view. Indeed a listener willing to understand a joke, is not necessarily inclined  to consider all the consequences of the violation of a norm. It is enough for a listener to find one norm that is violated, independently of the existence of sets of formulas that entail this norm. 
Moreover, as far as we know the contraction by default rules has not been much studied.}, however, when $\rho\in \Delta$, the removal of $\rho$ is in fact not mandatory  for checking incongruity, since the only thing that we would like to know is whether $\rho$ is violated in the context $\alpha\wedge \beta$. When $\rho\in P$, for showing that $\rho$ is violated we should at least remove $\rho$ from $P$ if we want to reason safely (without inconsistency). 

The removal of $\rho\in P$ by mere erasing of $\rho$ from $P$, may have some side effects due to proposition \ref{prop} which forbids a listener to consider formulas that are entailed by other formulas of $P$ as violated in the sense of the first condition of Definition \ref{disdef}. It appears that when this side effect takes place this could be understood as the fact that the listener is reluctant to the intended incongruity of the joke because she or he is not able to apprehend all the inconsistency it creates. When this effect does not occur the listener is able to handle the violation thus to appreciate the joke. This is the case in the two examples below where the norm violated cannot be derived from the rest of the knowledge base. The first one does not involve any default rule and illustrates the violation of a norm represented by a propositional formula. The second one concerns the violation of a norm expressed by a default rule.


\begin{exs}{elephant}
In this example, $\alpha= i\wedge e$, $\beta=e\wedge tt$, 
$\Delta$ is empty, and R2 is taken as the only norm, hence $\varphi \entails^m_\Sigma \psi$ reduces to $P\cup\{\varphi\}\models \psi$ when $P \cup \{\varphi\}$ is consistent.
We have $\{i\wedge e \wedge tt\} \cup P\setminus \{R2\}\models e \wedge \neg h$. 
Moreover, $P\setminus \{R2\}=\{R1,R3,R4\}$ is consistent with $i \wedge e\wedge tt$ and there is only one model satisfying them namely $\omega_0=$ i e tt $\neg$h in which e $\wedge$ $\neg$h holds, negating $\rho=R2=e\to h$. Hence the statement $(i\wedge e, tt\wedge e)$ is violating $R2$. 

Now, the set of models of $tt\wedge e$ satisfying $(P\setminus \{R2\})$ reduces to $\{\omega_0\}$, and $\omega_0\models i\wedge e$, hence the statement $(i\wedge e, tt\wedge e)$ is a revelation in disregard of $R2$ for a listener corresponding to $\Sigma$.
\end{exs}

\begin{exs}{fire} In this example, R1 and R3 are taken as norms, $\alpha=fh$ $\wedge$ sFn, $\beta=$ fh $\wedge$ an. We are going to show that {fh $\wedge$ sFn $\wedge$ an  $\entails^{bo}_{\Sigma\setminus \{R3\}}$ an $\wedge$ ir}. Indeed after removing the rule R3, the rules R1 and R4 have the same priority level in $\Delta$. Note that among the interpretations satisfying fh $\wedge$ sFn $\wedge$ an,    the interpretation {$\omega_1=$ fh an ir dh sFn} is the only one that does not violate any rule, hence it is the only preferred one. {$\omega_1\models$ an $\wedge$ ir}, thus violating $\rho= R3=$ an $\leadsto \neg$ ir. 
    
When removing $R3$ the only preferred model is $\omega_1$ which satisfies $sFn$, hence fh $\wedge$ an$\entails^{bo}_{\Sigma\setminus \{R3\}}$ fh $\wedge$ sFn which means that (fh $\wedge$ sFn, fh $\wedge$ an) is a revelation in disregard of $R3$. 
\end{exs}

As already suggested in introducing incongruity as a supplementary ingredient of humor, the two examples of subsection \ref{jokes} that where only considered as potentially funny (since being surprising and revealing), may also be seen as involving incongruity. 
Indeed, in example \ref{medecinmeme} we could underline the incongruity of the situation (an injured doctor) by the fact that usually doctors are always ready to cure people, which could be encoded by a rule such as  $doctorHimself\leadsto \neg injured$. 
In example \ref{femmechezmedecin}, there is an implicit social principle that is disregarded, namely ``a man should not follow a lady that he does not know (especially in a doctor’s office)’’. More precisely, this rule could be encoded as  $\neg knowEachOther \leadsto \neg together$ which is violated by $\beta$. 

\begin{exs}{femmechezmedecin} If we add the norm $\rho=\neg k \leadsto \neg t$ to $\Delta$ then the statement $(t,r)$ is still surprising, then the statement ($r, t \wedge \neg k$) is violating $\rho$ and it is a revelation in disregard of $\rho$. Thus the story $(t,r,t\wedge \neg k)$ is felt both incongruous and surprising by a listener represented by $(P,\Delta\cup \{\rho\})$.
\end{exs}

As we are going to see in the following remark, using the archetypal example of nonmonotonic reasonning, an exception to a default rule could be understood as incongruous to the extent that this default rule can be regarded as a norm. This is natural since when a new situation is presented by a teller, the listener is first inclined towards considering the situation as normal. 

 \begin{remark} Given a knowledge base $\Sigma=(P,\Delta)$ where $P=\{p\to b\}$ and $\Delta= \{R1=b\leadsto f, R2=p\leadsto \neg f$, $R3=b \wedge a \leadsto p\}$, the statement $(b,p)$ is potentially funny (as shown in Proof of Proposition \ref{propcourte}) and it is incongruous because a typical rule R1 is violated: $p\wedge b\entails^{bo}_{\Sigma\setminus\{R1\}} b \wedge \neg f$. Indeed we are in a case of an exception to rule R1. 
\end{remark}

In this subsection, we have provided a first attempt at modeling incongruity in jokes. We have used the setting introduced in Section \ref{formalisation} where the listener is identified with a a pair $(P,\Delta)$ acknowledging her / his capability of reasoning under incomplete knowledge, leading us to depart from a pure abstract  belief revision framework. This proposal remains preliminary, and there may be different ways to  extend this setting. For instance, we may think about introducing degrees of incongruity, we could also take into account the mixed emotions of the listener such as the shame associated to the violation of norms. Let us also admit that identifying what is a norm may depend from a listener to another this echoes the fact that the same joke cannot be understood in the same way by people with different background.

\subsection{Discussion on the relevance of the proposed model}
Without any doubt, a large number of jokes rely on a  surprising punchline in the sense that the teller presents an incomplete situation that puts the listener on a wrong track at the beginning of the joke, before the punchline reveals the reality of the situation. The fact that such a mechanism is present in many jokes is acknowledged by authors specialized in the study of jokes such as in \cite{Ritchie2018} who himself outlined a modeling of the kind discussed in Section \ref{formalisation}, although in a much less formal way  (\cite{Ritchie1999,Ritchie2002}).

Nevertheless, it is true that there are circumstances other than jokes where we can be surprised at the end, as in detective stories for instance, which are not especially funny in general. 
Indeed in this paper we speak of  ``potentially funny'' jokes, since on the one hand not all listeners laugh at the same jokes, and on the other hand the triggering of laughter also depends in part on the teller. Laughter is a matter of cognitive science which also involves physiological, psychological, sociological and cultural dimensions. 

Still, surprise and revelation remain ingredients that are at work in many jokes as noticed by philosophers and psychologists. But there are other kinds of surprising situations that are not based on incomplete information at the beginning, but rather rely on 
 inconsistency,  as in the  elephants joke.
Then, to escape inconsistency, one has to  ``act  as if'' inconsistency was not there (a known recipe for comical effects), we then speak of incongruity. 
Incongruity can be reinforced if the inconsistency involves a socially forbidden situation,  
for example,  setting the neighbor's house on fire, or 
undressing in front of strangers. 

It is difficult to determine what characterizes the specificity of what leads to laughter, and over the centuries philosophers have only been able to identify a few ingredients, the three main ones being the surprise effect, the incongruity and the superiority. 
Indeed all the jokes considered here exhibit some    incongruity and revelation (and possibly some surprise due to incomplete information), even if  there exist jokes which may be also based on hazardous analogies (see  \cite{DuPr2022} for some examples).

Whatever the case (surprising punchline or incongruity) the joke requires a resolution, this is its revealing side. This modest contribution has proposed  a  formalization of the two ingredients. Interestingly enough, they are directly related to two basic issues in artificial intelligence: incomplete information and inconsistency.



\section{Conclusion}\label{Con}
In this article, dedicated to Philippe Besnard, we have revisited our proposal to see jokes as a matter of surprise and revelation expressed in terms of belief revision. The key ideas of this paper are i) the claim that the capability of the listener of reasoning under incomplete information is necessary for understanding a joke 
 and ii) the notion of incongruity can be integrated in a nonmontonic reasoning framework. Incongruity is seen as a violation of a social norm, the comic effect comes from the fact that the story teller acts as if this norm is not existing. 

This preliminary proposal could still be enriched in different ways. First, one may consider that the comic effect is all the greater as the disregarded social norm is more obvious or mandatory. This means that one may introduce some level of incongruity in relation with the shocking nature of the non-compliance with the involved law. 

As recalled in the introduction, another acknowledged ingredient of humor is superiority. Taking this notion into account would require a more sophisticated framework for representing the communication process between agents that could reason and speak about other agents. This calls for a genuine multiple agent modeling. Such a framework would allow for representing how a teller can take into account what he or she knows about the listener and how the teller can create a sense of complicity with the listener, possibly against a common object of mockery. 
In such a context, principles of good communication (such as \cite{Grice1975}'s maxims) 
can be at stake. This allows for jokes that are incongruous with respect to such principles (like absurd jokes that disregard the principle of always taking the reality into account). 

In general, understanding a story is facilitated by the ability to visualize and put oneself in the place of the characters. This helps to perceive the blatant side of the violation of the law. Imagining a character in a story capable of breaking an obvious law may allow the listener for a kind of ``comic release’’ (from the pressure to follow that law). At the opposite, a story in a universe with people you can hardly identify with has little chance of making you laugh. However, in this context,  it is possible to achieve laughter by helping the listener to identify with the characters. In \cite{Valitutti2016}, geometric forms are moving on a screen where a small circle seems to follow a big one. This scene is shown to an audience that do not laugh at it. The same scene is shown with a fictional  dialog between the shapes involving a man and a woman, then it provokes the laughter of the audience. 
Lastly, a story itself may be more or less funny, but the way of telling it is also important (it may involve incongruity, since the storyteller might disregard some Gricean-like maxims).

In the same way as surprise and revelation are not enough for uniquely characterizing jokes, it is not sure that the addition of incongruity would be sufficient for capturing the comic effect of the joke and enough for characterizing it in a unique way. Indeed \cite{Spencer1860} emphasized that incongruity is not necessarily a source of laughter since it can also provoke wonder, as sometimes in art. 

\section*{Acknowledgements} This paper has greatly benefited from the remarks of the reviewers. 

\bibliography{biblioJoke}
\bibliographystyle{apacite}
\end{document}